\documentclass{llncs}

\usepackage{amssymb,amsmath}
\usepackage{graphicx}
\usepackage{url}
\usepackage{relsize}
\usepackage{xfrac}
\usepackage{color}
\let\polishl\l

\newcommand{\Wb}{\boldsymbol{W}}

\renewcommand{\P}{\boldsymbol{P}}

\newcommand{\budget}{B}
\newcommand{\eigen}{\lambda}
\newcommand{\Eigen}{\lambda}

\DeclareMathOperator*{\argmin}{argmin}
\DeclareMathOperator*{\argmax}{argmax}

\newcommand{\x}{\boldsymbol{{x}}}
\definecolor{blue}{rgb}{0,0,0.8}
\definecolor{magenta}{cmyk}{0,1,0,0}
\definecolor{lightgrey}{rgb}{0.5,0.5,0.5}
\definecolor{DarkGreen}{rgb}{0.0, 0.5, 0.0}

\newcommand{\wh}{\hat{\w}}

\renewcommand{\v}{{\boldsymbol{v}}}

\newcommand{\ssize}{\scriptsize}
\renewcommand{\S}{{\mathbb{S}}}

\def\smallf #1{\text{\ssize  {#1}}}

\newcommand{\X}{{\boldsymbol{X}}}
\renewcommand{\L}{{\mathbb{L}}}
\newcommand{\R}{{\mathbb{R}}}
\newcommand{\Z}{{\boldsymbol{Z}}}

\newcommand{\Y}{{\boldsymbol{Y}}}

\newcommand{\w}{{\boldsymbol{w}}}

\renewcommand{\d}{{\boldsymbol{d}}}

\renewcommand{\l}{\ell}
\renewcommand{\ll}{{\boldsymbol{\ell}}}
\newcommand{\y}{\boldsymbol{y}}

\newcommand{\C}{\boldsymbol{C}}

\newcommand{\E}[2][]{\mathbb{E}_{#1}\left[\;#2\;\right]}

\newcommand{\I}{{\boldsymbol{I}}}

\newcommand{\yh}{\widehat{\y}}

\renewcommand{\a}{\boldsymbol{a}}

\newcommand{\e}{\boldsymbol{e}}

\newcommand{\tr}{{\mathrm{tr}}}

\newcommand{\half}{{\mathsmaller{\frac{1}{2}}}}

\newcommand{\Yh}{\widehat{\Y}}

\urldef{\mailsa}\path|niejiazhong@soe.ucsc.edu|
\newcommand{\keywords}[1]{\par\addvspace\baselineskip
\noindent\keywordname\enspace\ignorespaces#1}

\begin{document}
\pagestyle{plain}

\title{Online PCA with Optimal Regrets \thanks{The first and the third authors were supported by the NSF grant
 IIS-0917397. 
The second author was supported by the Fundation for Polish Science under the Homing Plus Program, co-financed by the European Regional Development Fund.}
}

\author{Jiazhong Nie\inst{1}
\and Wojciech Kot{\polishl}owski\inst{2}
\and Manfred K. Warmuth\inst{1}}

\institute{Department of Computer Science, University of California,
Santa Cruz, CA 95064\\
\email{niejiazhong|manfred@cse.ucsc.edu} \and
Institute of Computing Science, Pozna{\'n} University of Technology, Poland\\
v\email{wkotlowski@cs.put.poznan.pl}
}
\maketitle
\begin{abstract}
We carefully investigate the online version of PCA, where in each trial
a learning algorithm plays a $k$-dimensional subspace, and
suffers the compression loss on the next instance when
projected into the chosen subspace. 
In this setting, we give regret bounds for 
two popular online algorithms,
Gradient Descent (GD) and Matrix Exponentiated Gradient (MEG). 
We show that both algorithms are essentially optimal in the worst-case
when the regret is expressed as a function of the number of trials. 
This comes as a surprise, since MEG is commonly believed 
to perform sub-optimally when the instances are sparse. 
This different behavior of MEG 
for PCA is mainly related to the non-negativity of the loss
in this case, which makes the PCA setting qualitatively different
from other settings studied in the literature. 
Furthermore, we show that when considering regret bounds as a function of
a loss budget, MEG remains optimal and strictly outperforms GD.

Next, we study a generalization of the online PCA problem, 
in which the Nature is allowed to play with dense instances, which are positive matrices with bounded largest eigenvalue. 
Again we can show that MEG is optimal and strictly better than GD in this setting.

\end{abstract}
\section{Introduction}
In \underline{P}rincipal \underline{C}omponent \underline{A}nalysis
(PCA), the $n$-dimensional data is projected / compressed onto a $k$-dimensional subspace
so that the total quadratic compression loss is minimized.
The problem of (centered) PCA is equivalent to finding
the eigenvectors of the $k$ largest eigenvalues of the covariance matrix. 
Here the data points $\x_t$ are arbitrary unit vectors in $\R^n$ and the
instances of the PCA problem are the outer products
$\x_t\x_t^\top$. The covariance matrix $\C = \sum_t \x_t\x_t^\top$ is the sum of the instances.

In this paper we consider the online version of centered PCA
\cite{pca}, where in each trial $t=1,\ldots,T$,
the algorithm chooses (based on the previously observed points $\x_1,\ldots,\x_{t-1}$)
a subspace of dimension $k$ described by a projection matrix $\P_t$ of rank $k$. 
Then a next point $\x_t$ (or instance $\x_t\x_t^\top$) is revealed 
and the algorithm incurs the ``compression loss'':
\begin{equation}
\label{e:loss}
\|\x_t - \P_{t} \x_t\|_2^2\;\;=\;\;\tr((\I-\P_t)\x_t\x_t^\top).
\end{equation}
The goal is to obtain an online algorithm 
whose cumulative loss over trials $t=1,\ldots,T$ is close to the cumulative loss of the best rank $k$ projection matrix chosen in hindsight after seeing all $T$ instances. 
The difference between the cumulative losses of the algorithm 
and the best off-line comparator is called the \emph{regret}. 

There are two main families of algorithms in online learning:
The Gradient Descent (GD)\cite{quadratic,gd} family
which is based on regularizing with the squared Euclidean distance,
and the Exponentiated Gradient (EG)\cite{eg} family
which use the relative entropy as their regularization.
The first family leads to \emph{additive updates} of the
parameter vector/matrix. When there are no constraints on
the parameter space, then the parameter vector/matrix is a linear
combination of the instances. 
However when there are constraints, then after the update 
the parameter is projected onto the constraints (via a Bregman
projection w.r.t. the squared Euclidean distance). 
As we shall discuss in the conclusions (Section \ref{sec:conclusion}),
projections w.r.t. inequality constraints introduce all
kinds of subtle problems for GD.
The second family leads to \emph{multiplicative update}
algorithms.
For that family the non-negativity constraints on the
parameters are already enforced and less projections are needed.

In \cite{pca}, a matrix version of the multiplicative update
was applied to PCA, whose regret bound is logarithmic in the
dimension $n$.
This algorithm is based on regularizing with the quantum relative entropy 
and is called the {\em Matrix Exponentiated Gradient} (MEG) algorithm
\cite{meg}. 
Beginning with some of the early work on
linear regression \cite{eg}, it is known that multiplicative updates are especially
useful when the instances are \emph{dense}.
In the matrix context this means that the symmetric
positive semi-definite instance matrix
$\X_t\in\R^{n\times n}$ processed at trial $t$ has maximum eigenvalue of say one.
However in the PCA context, the instance matrices are
the outer products, i.e.  $\X_t=\x_t \x_t^\top$.
Such instances (also called \emph{dyads}) are \emph sparse in the sense that their trace
norm is one, independent of the dimension $n$ of the instance matrix. 
Thus, one may suspect that MEG is not able to fully exploit the sparsity of the instance
matrices. 
On the other hand for linear regression, GD is known to have the advantage
when the instances are sparse and consistently with that,
when GD is used for PCA, then 
its regret is bounded by a term that is \emph{independent} of
the dimension of the instances. 
The advantage of GD in the sparse case is also
supported by a general survey of Mirror Descent algorithms
(to which GD and MEG belong) for the case when the loss
vectors (which have negative components)
lie in certain symmetric norm balls \cite{md}.

Surprisingly, the situation is quite different for PCA: We 
show that MEG achieves the same regret bound as GD for
online PCA 
(despite the sparseness of the instance matrices)
and the regret bounds for both algorithms
are within a constant factor of our new lower bound that
holds for any online PCA algorithm.
This surprising performance of MEG comes from the fact
that the losses in the PCA case are restricted to be non-negative, 
and therefore our results are qualitatively different
from the cases studied in \cite{md} where loss vectors
are within a $p-$norm ball, i.e. symmetric around zero. 

Actually, there are two kinds of regret bounds in the
literature: bounds expressed as a function of the time horizon $T$
and bounds that depend on an {\em upper bound} on the loss
of the best comparator (which we call a {\em loss budget}
following \cite{jakeminimax}).
In typical applications for PCA, there exists a low dimensional subspace
which captures most of the variance in the data
and guarding against the worst-case loss that grows with
$T$ is not useful.
We can show that when considering regret bounds as a function
of a loss budget, MEG is optimal and strictly better than GD
by a factor of $\sqrt{k}$.
This suggests that the multiplicative updates algorithm is the best
choice for prediction problems, 
in which the parameters are mixture of projection matrices 
and the losses are non-negative.
Note that by upper bounds on the regret, we mean upper
bounds for particular algorithms. However, the matching
lower bounds are always proved against any algorithm that
solves the problem.

\subsubsection*{Related work and our contribution:}
The comparison of the GD and MEG algorithms has quite an
extensively history (see, e.g. \cite{eg,span,mddual,md}). 
It is simplest to compare algorithms in the case
when the loss is linear. Linear losses are the least convex
losses and in the regret bounds,
convex losses are often approximated by first-order
Taylor approximations \cite{gd} which are linear,
and the gradient of the loss functions as the loss vector. 
Note that in this case the assumptions on the gradient of the loss
are typically symmetric.

In the case when the parameter space and the space of loss
vectors are convex and symmetric, the regret bounds are as
expected: EG is optimal when the parameter space is 1-norm
bounded and the loss vectors are infinity norm bounded,
and GD is optimal when the both spaces are 2-norm bounded
\cite{mddual,md}.
However, none of the previous work exploits the special PCA setup,
where the loss matrices (here the instances) are non-negative
and sparse (see \eqref{e:loss}). In this paper we carefully study this case.

We also made significant technical progress on the lower
bounds. Previous lower bounds focused on the 
non-sparse case \cite{pca,chedge}. Lower bounds were proved as a function of a 
loss budget.
In this paper we prove lower bounds as a function of time. 
These lower bounds harbor the budget case as a special case.

For the time dependent case, lower bounds were previously
shown for the expert setting \cite{book,advice,sto}.
However, these lower bounds rely on
the Central Limit Theorem and only hold in the limit
(as $T,n\rightarrow \infty$). 
In contrast our lower bounds use a probabilistic bounding
argument for the minimum of $n$ random variables and the
resulting bounds are non-asymptotic.

In summary, our contribution consists of proving tighter regret bounds
for two online PCA algorithms, as well as
proving lower bounds on the regret of any algorithm
for online PCA. From that we get the following conclusions:
MEG's and GD's regret bounds are independent of the dimension $n$ of the
problem and are tight within a constant factor
when the time horizon $T$ is fixed, which implies that both
algorithms are essentially optimal in this case.
If we fix the loss budget instead, MEG remains optimal,
while GD is proved to be suboptimal.

Furthermore, for a generalization of the PCA setting to the dense
instance case, we improve the known regret bound significantly 
by switching from a loss version to a gain version of MEG.
It turns out that MEG is optimal in the dense setting as
well, whereas GD is not.

\subsubsection*{Outline of the paper:}
In Section \ref{sec:online} we describe the MEG and GD
algorithms and prove regret bounds for them.
In Section \ref{sec:lower_bounds} we prove
lower bounds for both the sparse and the dense setting.
We conclude with an open problem about the
Incremental Off-line version of GD.

\iffalse
\subsubsection*{Definitions and notations used throughout:}
\begin{itemize}
\item $T$ is the length of the instance sequences. 
Particular trails are indexed by $t$.
\item $n$ is the dimension of parameter and instance spaces,
i.e. the dimension of the matrices or the number of experts
in the expert setting.
\item $k$ is the number of principal components
or the rank of the projection matrices.
\item $m$ is the size of the complementary subspace, i.e. $m= n-k$.
\item $p$ denotes the norm of the instances 
($p=1$ in PCA, but we also consider other cases).
\item $\budget$ is an upper bound on the loss of the best
comparator for the given instance sequence. The bounds is
used to tune the learning rates of the algorithms.
\item $\mathcal{W}$ is the parameter matrix; in particular $\mathcal{W}_n^m$ denotes the set of capped $n$-dimensional density matrix with trace $m$ and eigenvalues bounded by $1$;
\item $\Wb$ denotes a capped density matrix;
\item $\x$, and $\x\x^\top$ denote the instance (counterpart of the loss vector in the expert setting);
\item $\C$ denotes the sum of the instance $\C = \sum_t \x_t \x_t^\top$ (counterpart of the cumulative losses in the expert setting);
\item $\eigen_1,\ldots,\eigen_n$ denote the eigenvalues of a given instance $\x\x^\top$, while $\Eigen_1,\ldots,\Eigen_n$ denotes the eigenvalues of $\C$.
\end{itemize}
\fi

\section{The online algorithms}
\label{sec:online}

The GD and MEG algorithms are both
examples of the \emph{Mirror Descent} algorithm \cite{md}. 
Mirror Descent updates its parameters by minimizing a trade-off between a 
divergence to the parameter at the end of the last
trial and the loss on the current single instance,
followed by a projection into the parameter set. 
The divergence is always a Bregman divergence.  
In Machine Learning these updates were discovered in
\cite{eg,hw}. If we choose the quantum relative entropy
as the Bregman divergence, 
then we get the matrix version of a multiplicative update algorithm known as 
\emph{Matrix Exponentiated Gradient} algorithm (here
denoted as MEG). 
Similarly, the squared Frobenius norm results in an additive update algorithm 
known as \emph{Gradient Descent} (GD).\footnote{We
avoided the name ``Matrix'' Gradient Descent, since
Gradient Descent updates are motivated by regularizing with
the squared Euclidean distance and the Frobenius norm of a
matrix is simply the Euclidean norm of ``vectorized'' matrix.}

{\bf Sparse and dense instances:}
We call a symmetric positive semi-definite matrix 
{\em sparse} iff its trace norm (sum of the eigenvalues) 
is at most one.
Note that the instance matrices in our online PCA setup are sparse 
since they are outer products of unit vectors.
We also generalize our subspace learning problem to
{\em dense} instance matrices, which are symmetric positive semi-definite matrices
with maximum eigenvalues at most one. 

\subsection{The MEG algorithms}
\label{sec:eg}
In the online PCA problem, the algorithm predicts at trial
$t$ with a projection matrix $\P_t$ of rank $k$ and incurs
the compression loss $\|\x_t - \P_{t} \x_t\|_2^2$ 
upon receiving the next point $\x_t$. 
This loss is equivalent to the linear loss $\tr\left((\I - \P_t) \X_t \right)$, 
where $\X_t=\x_t \x_t^{\top}$ is the instance matrix.
Actually, $\I-\P_t$ is a complementary projection matrix
which has rank $m=n-k$. Since 
$$\tr\left((\I - \P_t) \X_t \right) 
= \tr\left(\X_t \right) - \tr\left(\P_t\X_t\right),$$
there are always two versions of
the algorithms: one that produces projection matrices
of rank $m=n-k$ and minimizes the compression loss
$\tr\left((\I - \P_t) \X_t \right)$
and one that produces projection matrices of rank $k$
and maximizes the gain $\tr\left(\P_t \X_t \right)$
(or minimizes $-\tr\left(\P_t \X_t\right)$).
As we shall see, for the MEG algorithm the loss and the
gain versions  (referred to as \emph{Loss MEG} and 
\emph{Gain MEG} throughout the paper) are different, whereas for GD, both versions collapse
to the same algorithm. 

We allow the algorithms to choose their projection matrix at
random. Thus the algorithms maintain a mixture of projection matrices 
of rank $k$ or $m=n-k$, respectively, as their parameter
matrix $\Wb_t$. These mixtures are 
\emph{generalized density matrices}, which are symmetric, positive definite matrices 
with eigenvalues upper bounded by $1$, and trace equal to
$k$ or $m$, respectively \cite{pca}.
We use $\mathcal{W}_k$ and $\mathcal{W}_m$ to denote the
parameter space of all such matrices, respectively.
Now we define the update of the Loss MEG and Gain MEG as follows:
\begin{align*}
&\text{Loss MEG:} \quad \Wb_{t+1} 
= \argmin_{\Wb \in \mathcal{W}_m} \;\left( \Delta(\Wb, \Wb_t)
+ \eta \: \tr(\Wb \X_{t}) \right),\\
&\text{Gain MEG:} \quad \Wb_{t+1} 
= \argmin_{\Wb \in \mathcal{W}_k} \;\left( \Delta(\Wb, \Wb_t) 
- \eta\: \tr(\Wb \X_{t}) \right),
\end{align*}
where $\Delta(\Wb,\Wb')=\tr\left(\Wb(\log \Wb - \log \Wb' \right)$ 
is the quantum relative entropy, and $\eta>0$ is a learning rate.
Also $\tr(\Wb \X_{t})$ for $\Wb\in \mathcal{W}_m$ is the
expected loss in trial $t$ 
of the random projection matrix of rank $m$ drawn 
from the mixture summarized by $\Wb\in \mathcal{W}_m$.
Similarly, $\tr(\Wb \X_{t})$ for $\Wb\in \mathcal{W}_k$ is
the expected gain in trial $t$ 
of the random projection matrix of rank $k$ drawn from the mixture summarized by $\Wb \in \mathcal{W}_k$,
Note that the loss version of MEG corresponds to the
original MEG algorithm 
developed in \cite{pca}, where it was shown to have the
following regret bound:
\begin{equation}
 \mathrm{regret}_{\mathrm{Loss}\;\mathrm{MEG}} \leq \sqrt{2 \budget m \log \frac{n}{m} } + m \log \frac{n}{m}.
 \label{eq:regret_bound_B_MEG}
\end{equation}
This bound holds for any sequence of instance matrices
(dense as well as sparse)
for which the total compression loss of the best rank $k$ subspace 
does not exceed the loss budget $\budget$.
With a similar analysis, 
the regret of Gain MEG can be bounded by
\begin{equation*}
 \mathrm{regret}_{\mathrm{Gain}\;\mathrm{MEG}} \leq \sqrt{2 kG \log \frac{n}{k} } .
\end{equation*}
This bound holds for any sequence of instance matrices (dense as well as sparse)
for which the total gain of the best rank $k$ subspace 
does not exceed the gain budget $G$.

Budget dependent upper bounds on the regret always lead to time
dependent regret bounds (as exploited in the proof of the below Theorem). 
Note that for PCA, the gain bound $G$ is usually much larger 
than the loss bound $B$ and therefore Gain MEG is not that
useful for PCA. However as we shall see for dense
instances, Gain MEG is sometimes better than Loss MEG.
Incidentally, for lower bounds on the regret the
implication is reversed in that time
dependent regret bounds imply budget dependent regret bounds.

\begin{theorem}
\label{th:T_MEG}
For sparse instance sequences of length $T$, the regret of the
Loss MEG and Gain MEG algorithms is upper bounded by:
\begin{align}
 \label{eq:regret_bound_T_MEG}&
\mathrm{regret}_{\mathrm{Loss}\;\mathrm{MEG}}  \;\leq\; m\sqrt{ \frac{2T}{n} \log \frac{n}{m} } + m \log \frac{n}{m}\;\leq\;\sqrt{ \frac{2kmT}{n} } + k \\[0.2cm]
& \mathrm{regret}_{\mathrm{Gain}\;\mathrm{MEG}}  \leq\sqrt{ 2kT \log \frac{n}{k} }. \nonumber 
\intertext{Similarly, for dense instances, the following
regret bounds hold:}  
& \mathrm{regret}_{\mathrm{Loss}\;\mathrm{MEG}}  \leq m\sqrt{ 2T \log \frac{n}{m} } + m \log \frac{n}{m}\nonumber \\
& \mathrm{regret}_{\mathrm{Gain}\;\mathrm{MEG}}  \leq k\sqrt{ 2T \log
\frac{n}{k} }.  \nonumber
\end{align}
\end{theorem}
\begin{proof}
The best rank $k$ subspace picks $k$ eigendirections of the covariance matrix $\C =\sum_{t=1}^T \X_t$ with the largest eigenvalues. 
Hence the total compression loss equals the sum of the smallest $m$ eigenvalues of $\C$. 
If $\Eigen_1,\ldots,\Eigen_n$ denote the eigenvalues of $\C$ then:
\[
\sum_{i=1}^n \Eigen_i = \tr(\C)= \sum_{t=1}^T \tr\left(\X_t\right) \leq \left\{\begin{array}{cc} \; T \quad &\text{for sparse instances,}\\[0.2cm] \; Tn \quad&\text{for dense instances.}\end{array}\right.
\]
This implies that the total compression loss of the
comparator is upper bounded by $\frac{T m}{n}$ and $Tm$, respectively. 
Plugging these values into \eqref{eq:regret_bound_B_MEG}
results in the bounds for Loss MEG.
The second inequality in \eqref{eq:regret_bound_T_MEG} follows from 
$$m \log \frac{n}{m} = m \log\left(1 + \frac{k}{m}\right) \leq k.$$
For the regret bounds of Gain MEG, 
we use the fact that $G$ is upper bounded by $T$ 
when instances are sparse and upper bounded by $kT$ when the instances are dense.\qed
\end{proof}
Note that in light of previous results for MEG, it is
actually surprising that the regret bound
\eqref{eq:regret_bound_T_MEG} 
for Loss MEG with sparse instances is independent of the dimension $n$ of the problem.

We now discuss in detail which version of MEG has a better regret
bound for the dense instance case.
We claim that this depends on the value of $k$, the
dimension of the chosen subspace.
Consider the ratio of the regret bounds of Loss MEG over
Gain MEG.
When $T\geq k$, then we can ignore the $m\log \frac{n}{m}$ term
in the Loss MEG bound since this term is at most $k$.
In this case the ratio becomes:
\[ \Theta\left(\sqrt{\frac{k^2}{m^2}\frac{\ln \frac{n}{k} }{\ln \frac{n}{m}}}\right).\]
When $k \leq \frac{n}{2}$, $\ln \frac{n}{m} =
\ln(1+\frac{k}{m}) = \Theta(\frac{k}{m})$, and the ratio simplifies to
$\Theta\left(\sqrt{\frac{ \ln\frac{n}{k}}{
\frac{n}{k}}}\right)$.
Therefore, when $\frac{n}{k}$ grows,
the regret bound for the Loss MEG is less than the regret
bound for the Gain MEG.
Similarly, when $k \geq \frac{n}{2}$, the ratio becomes 
$\Theta\left(\sqrt{\frac{ \frac{n}{m}}{ \ln\frac{n}{m}}}\right)$ 
and the regret bound for the Gain MEG is better in this case.
\subsection{The GD algorithm}
\label{sec:GD}
In this section we consider the GD algorithm (see e.g.\
\cite{gd1,gd}) which is motivated by the squared Frobenius norm
(The loss and gain versions are the same in this case and
we use the loss version below):
\[
\Wb_{t+1} = \argmin_{\Wb \in \mathcal{W}_m}  
\;\:\left( \half\|\Wb -\Wb_t\|_F^2  + \eta\: \tr(\Wb \X_{t})\right).
\]
This algorithm is simple and a time dependent regret
bound has been proved for arbitrary convex losses \cite{gd,md}.
By applying this bound to PCA we obtain:
\[\mathrm{regret}_{GD} \leq \left(\!\max_{1\leq t \leq T}
 \|\X_t\|_F\right)  \sqrt{T \|\Wb_1-\Wb_c\|_F^2}
=  \left(\!\max_{1\leq t \leq T}
\|\X_t\|_F\right)\sqrt{\frac{mkT}{n}},
\]
where $\Wb_c$ is any comparator in the parameter space $\mathcal{W}_m$.
For sparse instances, $\|\X_t\|_F = \sqrt{\tr(\X_t\X_t^\top)}\!\leq\!1$, 
the regret is bounded by $\sqrt{mkT/n}=\sqrt{(n-k)k T/n}$.
This is the same as the regret bound for Loss MEG 
\eqref{eq:regret_bound_T_MEG} except for an additional $\sqrt{2}$ factor
bound for the Loss MEG.
When instances are dense, $\|\X_t\|_F \leq n$, resulting in
regret bound of $\sqrt{m kT}$.
To see that this bound is worse than the MEG bound 
for dense instances, we can consider the ratio of the
regret bound for GD over the regret bound for the
appropriate version of MEG. 
It is easy to check that when $k\leq \frac{n}{2}$, 
the ratio is $\Theta(\sqrt{\frac {\frac{m}{k}}{\ln(\frac{2m}{k})}})$, 
and when $k\geq \frac{n}{2}$, the ratio is 
$\Theta(\sqrt{ \frac {\frac{k}{m}} {\ln(\frac{2k}{m})}})$. 
In both case, the regret bound for MEG is better by more
than a constant factor.

We now conclude this section by investigating budget bounds for GD.
Since GD achieves the same time horizon dependent regret
bound as Loss MEG, 
we first conjectured that this is also the case for budget
dependent regret bounds.
However, this is not true: 
we will now show in remainder of this section a $k\sqrt{B}$ \emph{lower bound} on
the regret of GD for instance sequences with budget $\budget$.
Since Loss MEG has regret at most $\sqrt{kB}$ in this
case, this lower bound shows that GD is suboptimal by a factor of $\sqrt{k}$.

It suffices to prove the lower bound on a restricted data set.
As already observed in \cite{pca}, the PCA problem has the \emph{$m$-set}
problem as a special case. In this problem,
all instance matrices are diagonal (i.e. the eigenvectors
are the standard basis vectors) and therefore the algorithm
can restrict itself to choosing subspaces that are subsets of the standard basis
vectors.
In other words, PCA collapses to learning a subset of
$m=n-k$ experts. The algorithm chooses a subset of $m$ out of $n$ experts in
each trial, the loss of all experts is given as a vector $\ll\in[0,1]^n$,
and the loss of a set is the total loss of all experts in the set.
The algorithm maintains uncertainty over the $m$-sets by means of a 
weight vector $\w  \in [0,1]^n$, such that $\sum_{i=1}^n w_i = m$.
We denote the set of all such weight vectors as $\S_m$. The
above GD
algorithm for PCA
specializes to the following algorithm for learning sets:

\begin{equation}
\label{e:diaggd}
\begin{array}{ll}
\text{Gradient Descent step:}\quad   & \wh_{t+1} = \w_t - \eta\ll_t \\
\text{Projection step: }             & \w_{t+1} = \argmin_{\w \in \S_m}  \|\w-\wh_{t+1}\|^2.
\end{array}
\end{equation}
The projection step is a convex optimization problem with
inequality constraints and can be analyzed 
using its KKT conditions.
We only describe the projection step in two cases needed
for the lower bound.
Let $\w_t=(w^1,\cdots, w^n)$ be the weight of GD at trial
$t$. Our lower bound is for the sparse case. In the set
problem this means that the loss vectors $\ll_t$ are standard basis vectors. 
Let $\ll_t=\e_{i_t}$. In the simplest case,
the descent step decreases the weight of expert $i$ by $\eta$ 
and the projection step adds $\frac{\eta}{n}$ to all $n$ weights
so that the total weight remains $m$:
\begin{equation}
\widehat{\w}_{t+1} 
\!=\! (w^1,\!\ldots\!, w^i\!-\!\eta,\!\ldots\!, w^n), 
\w_{t+1}
\!=\! \Big(\!w^1 \!+\!\frac{\eta}{n},\!\ldots\!, w^i 
      \!-\!\frac{(n-1)\eta}{n},\! \ldots\!, w^n\!+\!\frac{\eta}{n}\Big).
\label{eq:gd1}
\end{equation}
Two problems may happen with the additive adjustment: 
$w_i - \eta+\frac{\eta}{n}$ might
be negative or some of the weights $w_j+\frac{\eta}{n}$
might be larger than 1.
The projection step is slightly more involved when this happens.
In our lower bound, we only need the following additional case:
$$\w_t=\big(\underbrace{w^1,\ldots, w^{i-1}}_{\text{all $\le 1-\frac{\eta}{n}$}},
\underbrace{w^i}_{<\frac{n-1}{n}\eta},
\underbrace{1-\delta,\ldots, 1-\delta}_{\text{for }\delta<\frac{w_i}{n}}\big).$$ 
One can show that in this case the projection
sets $w^i$ to 0, it sets the $n-i$ weights of size $1-\delta$ to
1, and it adds $\frac{w^i-(n-i)\delta}{i-1}$ to the first $i-1$
weights which are not capped. That is, in this case the projections
produces the following updated weight vector:
\begin{equation}
\label{eq:gd2}
 \w_{t+1} = \Big(w^1 +\frac{w^i -(n-i)\delta}{i-1},\ldots,w^{i-1} +\frac{w^i -(n-i)\delta}{i-1},
\!\!\!\!\!\!
\underbrace{0}_{\text{ capped at $0$}},
\!
\underbrace{1,\ldots,1}_{\text{capped at $1$}}\Big).
\end{equation}
Note that the total weight of the projected weight vector $\w_{t+1}$ is again $m$.

Now we are ready to give our regret lower bound for GD.
\begin{theorem}
For any $k \leq n/2$ and any learning rate $\eta$, 
there is a sparse loss sequence for online PCA which has budget $B$ 
and forces the GD algorithm \eqref{e:diaggd} to incur regret at least $\Omega(k\sqrt{B})$.
\end{theorem}
\begin{proof}
W.l.o.g.,
assume all the experts have the same initial weight $m/n=(n-k)/n\geq 1/2$.
Call the first $k$ experts \emph{bad experts}, 
the $(k+1)$st expert the \emph{faulty expert} 
and the last
$m-1$ experts \emph{good experts}.
Let $\eta'=\min\{\eta,1\}$.

The loss sequence consists of two phases.
We will show that the algorithm suffers loss at least 
$\Omega(\frac{k}{\eta})$ in the first phase
and essentially loss at least $B+\Omega(kB \eta)$ in the second
phase. The optimum trade-off between these two
term give the lower bound.

More precisely, in the first phase unit losses are given to bad experts
and in each trial the algorithm suffers the current weight of the chosen bad expert.
The phase ends when each of the good experts and the faulty expert have weights at least $1-\frac{\eta'}{4m}$. 
To show that the algorithm suffers loss at least $\Omega(\frac{k}{\eta})$ in this phase,
first notice that for a particular bad expert, 
its weight decreases by at most $\frac{n-1}{n}\eta$ when it receives 
a unit of loss and increases in all of the other trials (see \eqref{eq:gd1} and \eqref{eq:gd2}).
So when it receives loss for the $s$-th time, its weight is lower bounded by 
\begin{equation}
\label{eq:gd_lower}
\max\left\{\frac{m}{n}-(s-1)\frac{n-1}{n}\eta,0\right\}  \geq \max\left\{\frac{1}{2}-(s-1)\eta,0\right\} 
\end{equation}
The sum of \eqref{eq:gd_lower} with $s=1,2,\dots$ is a lower bound 
on the loss of the algorithm when this particular bad expert incurred a unit of loss.
Note that $\eqref{eq:gd_lower}$ is the term of an arithmetic series 
that is capped from below by zero. 
One can show that as long as there is a constant gap between the first and last term 
of the summation, then the sum of these terms
is at least $\Omega(\frac{1}{\eta})$.
In our case, this gap is at least $1/4$ since the first term (initial weight) is at least $1/2$ and the last term, upper bounded by this bad expert's weight after phase one, is less than  
\[ \underbrace{m}_{\text{\begin{tabular}{c}sum of the
weights \\of all experts\end{tabular}}} -
\underbrace{m(1-\frac{\eta'}{4m})}_{ \text{\begin{tabular}{c}sum of the weights \\of faulty and
good experts\end{tabular}}} \leq \frac{1}{4}. \]
Since we have $k$ bad experts, we obtained a $\Omega(\frac{k}{\eta})$ 
lower bound on the loss of the algorithm during the first phase.

We now describe the second phase which lasts for $B$ rounds, 
where each round consists of several trials.
At the beginning of each round, 
the faulty expert receives one unit of loss. Its weight
after the gradient descent step is at most $1-\eta$ and
after the projection step it can be shown to be at most 
$\max\{1 -  \eta/2,0\}$ (see \eqref{eq:gd2}).
Notice that after the first trial of each round, 
all good experts will have weight 1 
since they are at least $1-\frac{\eta'}{4m}$ after phase one.
In the following trials of this round, unit losses are given to bad experts 
until the faulty expert recovers its weight to at least $1-\frac{\eta'}{4m}$.
Since all the weights of good experts are capped at $1$,
the re-balancing of weights only occurs between
the faulty and the bad experts. This means that in each trial, 
the faulty expert can only recover at most $1/(k+1)$ of the
loss incurred by the algorithm in this trial.
Thus we can lower bound the loss of algorithm in a given round as follows:
\[
1- \frac{\eta'}{4m}+(k+1)\Big(1 - \frac{\eta'}{4m}- \max\Big\{1- \frac{\eta}{2},0\Big\}\Big) \geq 1-\frac{\eta'}{4m}+\frac{k+1}{4}\min\{\eta,1\} = 1+\Omega(k\eta').
\]
After $B$ such rounds, algorithm suffers loss at least $B+\Omega(kB\eta')$:
When $\eta \geq 1$, this is $B+\Omega(kB)$ and when $\eta
\leq 1$, then summing up the bounds of the two phases,
gives an $\Omega(k/\eta) + B +\Omega(B k\eta)$ lower bound on the loss of the algorithm.
The latter is minimized at $\eta=\Theta(1/\sqrt{B})$ and for this choice
of $\eta$, the algorithm suffers loss at least $B+\Omega(k\sqrt{B})$.
The theorem now follows, since the best off-line $m$-set for the loss sequence consists 
of the faulty expert, which suffers total loss $B$, and all $m-1$
good experts, which incur no loss.
\qed
\end{proof}

\section{Lower bounds and optimal algorithms}
\label{sec:lower_bounds}
In the previous section, we showed a lower bound on the regret of GD
as a function of the budget $B$ of the sequence. 
In this section we show regret bounds for any algorithm
that solves the problem. In particular, we show regret lower bounds 
for online PCA and its generalization to the dense instance case.
As argued in Section \ref{sec:GD}, it suffices to prove our lower bounds for the $m$-set 
problem which is the vector version of online PCA and its
generalization to dense instances.
We prove lower bound on the minimax regret, 
i.e. the minimum worst case regret any algorithm can
achieve against the best set:
\[
 \min_{\ssize{  \begin{tabular}{c}  alg. A with \\  $\w_t \in \S_m$  \end{tabular}}}
\;\;
 \max_{\ssize{\begin{tabular}{c}  sparse/dense loss seq. \\  $\ll_{1 \dots T}$ of length T  \end{tabular}}} 
\quad\quad \underset{\ssize{\begin{tabular}{c}loss of alg. A 
- loss of best set \\
on loss sequence $\ll_{1 \dots T}$\end{tabular}}}{R(A,\ll_{1\dots T})}.
\]
Recall that $\S_m$ were vectors in $[0,1]^n$ of total
weight $m$ that represent mixtures of sets of size $m$.
Our lower bounds will match the uppers bounds on the regret
of MEG (within constant factors) that we proved in the previous section 
for both online PCA and its generalization to dense instances.
The lower bounds rely on the following
probabilistic bounding technique for the minimax regret:
\begin{align}
 \!\!\min_{\ssize{  \begin{tabular}{c}   alg. A with \\  $\w_t \in \S_m$ \end{tabular}}}
  \!\max_{\ssize{\begin{tabular}{c} 
loss seq.  $\ll_{1\dots_T}$  \\ of length T  \end{tabular}}} \!\!\!R(A,\ll_{1\dots T}) \; 
 &\geq   \!\!\!\!\min_{\ssize{  \begin{tabular}{c}   alg. A with \\
$\w_t \in \S_m$\end{tabular}}} \!\!\E[\ll_{1\dots T}\sim \mathcal{P}]{  R(A,\ll_{1\dots T})} \nonumber \\
\label{eq:lower1} & =\!\!\!\!\min_{\ssize{  \begin{tabular}{c}
alg. A with \\  $\w_t \in \S_m$ \end{tabular}}} \!\!\E[\ll_{1\dots T}\sim \mathcal{P}] {L_A} -   \E[L\sim \mathcal{P}]{L_C},
\end{align}
where $\mathcal{P}$ is any distribution on loss sequences, and
$L_A$ and $L_C$ are the cumulative losses of the algorithm and 
the best $m$-set, respectively.

\subsection{Lower bounds for PCA with sparse instances}
\label{sec:PCA_setting}
Recall that for the vectorized version of PCA, the loss
vectors $\ll_t$ are restricted to be standard basis vectors.
We start with the following technical lemma for two experts.
\begin{lemma}
\label{lem:ex}
Let $p \in [0,1]$ be such that 
$p\leq 1/4$ and $Tp\geq 1/2$. Assume that in a two expert
game, one of the experts is randomly chosen to suffer one unit of loss
with probability $2p$ in each trial, 
and with probability $1-2p$ none of the experts suffers any loss. 
Then, after $T$ independent trials, 
\[\E{ \text{Loss of the winner}}  \leq\;  Tp - c\sqrt{Tp},
\]
for a constant $c$ independent of $T$ and $p$.
\end{lemma}
Due to the space limit, we omit the proof of this lemma.
We are now ready to prove a lower bound for PCA. 
We first consider the case when $k\le \frac{n}{2}$.
\begin{theorem}
\label{th:sparse1}
For $T\geq k$ and $k\leq \frac{n}{2}$, 
in the $T$ trial online PCA problem with sparse instances, 
any online algorithm suffers worst case regret at least
$\Omega(\sqrt{kT})$.
\end{theorem}
\begin{proof}
At each trial, a randomly chosen expert 
out of the first $2k$ experts receives a unit of loss. 
To show an upper bound on the loss of the comparator, 
we group these $2k$ experts into $k$ pairs and notice that the losses of each expert pair have a joint distribution as
described in Lemma \ref{lem:ex} with $p=\frac{1}{2k}$.
Hence, the expected loss of the winner in each pair is at most $T/2k-c\sqrt{T/k}$, 
and the total expected loss for the $k$ winners from all $k$ pairs
is upper bounded by $T/2 - c\sqrt{kT}$.
Since the last $n-2k$ experts are loss-free, 
this is also an upper bound on the expected loss of the comparator,
because the comparator will pick $n-2k$ loss-free experts and $k$ best experts among the remaining $2k$ experts. 
On the other hand, since losses are generated independently
between trials, any online algorithm suffers loss at least $T/2$.
Taking the difference between two bounds concludes the proof. \qed
\end{proof}
Noting that $m\sqrt{\ln (n/m)T/n)} \leq  m\sqrt{(k/m)T/n} \leq \sqrt{kT}$, 
the lower bound in Theorem \ref{th:sparse1} matches the upper bound of Loss MEG algorithm in Theorem \ref{th:T_MEG} for the case $k\leq \frac{n}{2}$.
For the case $k\geq \frac{n}{2}$,
we need the following lemma, which 
is a generalization of Lemma \ref{lem:ex} to $n$ experts. 
In the proof we upper bound the minimum loss of the experts 
by the loss of the winner of a tournament among the experts. 
The tournament winner does not necessarily have the lowest
loss. However as we shall see later, its expected loss is close enough 
to the expected loss of the best expert to make this bounding techniques useful
for obtaining lower bounds on the regret.
\begin{lemma}
\label{lem:log}
Choose any $n, S$ and $T$, such that $n=2^S$ and $S$ divides $T$. 
If the loss sequence of length $T$
is generated from a distribution $\mathcal{P}$, such that:
\vspace{-0.3cm}
\begin{itemize}
\item at each trial $t$, 
the distribution of losses on $n$ experts is exchangeable,
\item the distribution of losses is i.i.d. between trials,
\end{itemize}
then,
\begin{align*}
&\E{\text{Minimum loss of $n$ experts in $T$ trials} } \\
&\qquad\qquad\leq S \;\E{\text{Loss of the winner among two experts in $T/S$ trials} }.
\end{align*}
 \end{lemma}
\begin{proof}
Due to the space limit, we only sketch the proof.
The key idea in the proof is to upper bound the loss of the
best expert by the loss of the expert that wins a {\em tournament}
with $S$ rounds. 
In each round, the experts are paired and compared against
their partners, using the sum of their losses in the next $T/S$ consecutive trails.
The winner of each local pair competition survives to the next round. 
The winners are again paired and winners among those continue
with the tournament until one expert is left from the original $n=2^S$ experts. 
We call this expert the tournament winner. 
The expected loss of the best expert is
upper bounded by the expected loss of the tournament
winner, which curiously enough
\newpage
equals the number of rounds times the
expected loss of the two expert case:
\begin{align*}
\E{\text{Minimum loss of all $n$ experts in all $T$ trials} }\!\!\!\!\!\!\!\!\!\!\!\!\!\!\!\!\!\!\!\!\!\!\!\!\!\!\!\!\!\!\!\!\!\!\!\!\!\!\!\!\!\!\!\!\!\!\!\!\!\!\!\!\!\!\!\!\!\!\!\!\!\!\!\!\!\!\!\!\!\!\!\!\!\!\!\!\!\!\!\!\!\!\!\!\!\!\!\!&\\
\leq\;\;\qquad\quad&\E{\text{\begin{tabular}{c}Loss of the
{\em tournament} winner in the\\ S rounds tournament among
the $n=2^S$ experts\end{tabular}}}\\
\stackrel{\smallf{since expectations sum}}{\text{\LARGE$=$}}\;&
\sum_{\text{rounds $1\le s\le S$}}\E{\text{Loss of {\em tournament} winner in round $s$}}\\
\stackrel{\smallf{i.i.d. loss btw. trials }}{\text{\LARGE$=$}}\;\;&S\;\E{\text{Loss of the  {\em tournament} winner in one round}}\\
\stackrel{\smallf{def. of local tournament}}{\text{\LARGE$=$}}&S \; 
\E{\text{Loss of winner among two experts in $T/S$ trials}}
.
\end{align*}
The last equality would be trivial if the distribution $\mathcal{P}$
on the sequence of loss vectors was i.i.d. between experts.
An additional argument is needed to show the equality with
the weaker assumption of exchangeability. 
\qed 
\end{proof}
We now consider the uncommon case when $k\ge \frac{n}{2}$:
\begin{theorem}
\label{th:sparse2}
For $T\geq n\log_2(n/m)$ and $k\geq \frac{n}{2}$,
in the $T$ trial online PCA problem with sparse instances, 
any online algorithm suffers worst case regret at least $\Omega(m\sqrt{\ln(n/m)T/n})$.
\end{theorem}
\begin{proof}
At each trial, a randomly chosen expert out of $n$ experts receives a unit of loss. 
To show an upper bound on the loss of the comparator, 
we partition the $n$ experts into $m$ groups and notice that the losses of the $n/m$ experts in each group are exchangeable. 
By applying Lemma \ref{lem:log} to each group, we obtain: 
\begin{align}
&\E{\text{Loss of the winner in a given group in $T$ trials} } \nonumber \\
\label{eq:lower2}&\quad\leq \log_2(\frac{n}{m}) \;\;\E{\text{Loss of winner of two experts in $\frac{T}{\log_2(\frac{n}{m})}$ trials} }.
\end{align}
We bound the last expectation that deals with the 2 experts
case by applying Lemma \ref{lem:ex} with $p=1/n$ and $T/\log_2(n/m)$.
This lets us replace the expectation by the upper bound
\[\frac{T}{\log_2(n/m)n} - \sqrt{\frac{T}{\log_2(n/m)n}}.\]
Plugging this into \eqref{eq:lower2} gives a $T/n - \sqrt{\log_2(n/m) T/n}$ upper bound on the expected loss of a winner in a given group.
We upper bound the expected loss of the comparator by the total loss of 
$m$ winners from the $m$ groups, 
which in expectation is at most $mT/n -  m\sqrt{\log_2(n/m) T/n}$.

Finally the loss of the algorithm is bounded as follows. 
Since every expert suffers loss $1/n$ in expectation at
each trail and losses are i.i.d. between
trials, any online algorithm suffers loss at least $mT/n$.
This concludes the proof. \qed
\end{proof}

Combining this lower bound with the upper bounds proved 
in Section \ref{sec:eg} on
the regret of Loss MEG for the sparse instance case 
results in the following corollary:
\begin{corollary}
For online PCA with sparse instances, the regret 
$\Theta(m\sqrt{\frac{T\ln(n/m)}{n}})$ 
of Loss MEG is within a constant factor of the minimax regret.
\end{corollary}

\subsection{Lower bound for PCA with dense instances}
\label{sec:dense_settings}
The following lower bound again employs Lemma \ref{lem:log}
which was proved using a tournament.
\begin{theorem}
\label{th:dense} 
For $T\geq \log_2(\frac{d}{\min\{k,m\}})$, in the $T$ trial online PCA
problem with dense instances, any online algorithm 
suffers worst case regret at least
\[
\Omega(m\sqrt{\ln(n/m)T}) \text{ when $m \leq \frac{n}{2}$
}\quad \text{or}\quad \Omega(k\sqrt{\ln(n/k)T}) \text{ when
$m \geq \frac{n}{2}$. }
\]
\end{theorem}
\begin{proof}
The proof is similar to the proof of Theorem \ref{th:sparse2}, 
except that at each trial, unit losses are independently given to all the experts with probability $\frac{1}{2}$.
For such a distribution over losses, any algorithm suffers cumulative loss 
at least $mT/2$ in expectation.
We now upper bound the comparator's expected loss by distinguishing two cases:
When $m\leq n/2$, we group the experts into $m$ groups and upper bound the 
comparator loss using the $m$ winners, one from each of the groups.
This gives an $mT/2 - cm\sqrt{\ln(n/m)T}$ upper bound, 
and results in a $\Omega(m\sqrt{\ln(n/m)T})$ lower bound for the regret.

When $m\geq n/2$, we group the experts into $k$ groups and consider a {\em loser} 
out of each group, i.e. the expert which suffers the {\em
largest} loss in each group.
One can flip around the content of Lemma \ref{lem:log} to show that the loser's loss in a group of $n/k$ experts is lower bounded by 
$ T/2 + c\sqrt{\ln(n/k)T}$, so that the expected loss of
all $k$ losers is lower bounded by $kT/2 + ck\sqrt{\ln(n/k)T}$.
The claimed regret bounds now follows from the fact that the cumulative loss of the comparator is upper bounded by the total expected loss of all experts ($nT/2$) minus
the total expected loss of all $k$ losers. 
This completes the sketch proof.\qed
\end{proof}
Combining this lower bound with the upper bounds on the
regret of Loss MEG and Gain MEG for dense instance case proved in Section \ref{sec:online} gives a following corollary,
which basically states that the Loss MEG is optimal for $m \leq \frac{n}{2}$ 
while the Gain MEG is optimal for $m \geq \frac{n}{2}$.
\begin{corollary}
Consider online PCA with dense instances.
\begin{itemize}
 \item When $m \leq \frac{n}{2}$, the regret $\Theta(m\sqrt{T\log \frac{n}{m}})$ of Loss MEG is within a constant factor of the minimax regret.
 \item When $m \geq \frac{n}{2}$, 
the regret $\Theta(k\sqrt{T\log \frac{n}{k}})$ of Gain MEG is within a constant factor of the minimax regret.
\end{itemize}
\end{corollary}

\paragraph{Minimax regret for sequences with a budget.} 
One can also show the minimax regret for a prediction game in which the budget $B$ is fixed, rather than the time horizon $T$.
In this case, no matter if the instances are dense or sparse,
we get the following corollary establishing the optimality
of Loss MEG:
\begin{corollary}
\label{cor:budget}
Let $B\geq m\log_2\frac{n}{m}$. For online PCA with both
sparse and dense instances, the regret $\Theta(\sqrt{m\ln(n/m)B})$ of Loss MEG is within a constant factor the minimax budget regret.
\end{corollary}
\begin{proof}
Since the instance matrices have eigenvalues bounded by one, 
the minimax regret is upper bounded by 
 $O(\sqrt{m\ln(n/m)B})$, the regret bound 
of the Loss MEG algorithm given in 
\eqref{eq:regret_bound_B_MEG}, Section \ref{sec:eg}.
On the other hand, we now reason that for any algorithm 
we can construct a sparse instance sequence 
of budget $B$ incurring regret at least $\Omega(\sqrt{m\ln(n/m)B})$.
This instance sequence is constructed via Theorem
\ref{th:sparse1} and Theorem \ref{th:sparse2}:
For any algorithm, these theorems provide 
a sparse instance sequence of length $T$ 
with regret at least $\Omega(m\sqrt{\frac{T\ln(n/m)}{n}})$.
We apply these theorems with $T=\frac{n}{m}B \geq n\log_2\frac{n}{m}$.
Since the produced sequence is sparse and has length $\frac{n}{m}B$, 
its budget is at most $B$.
Finally plugging $T=\frac{n}{m}B$ into the regret bounds
guaranteed by the theorems results in the regret 
$\Omega(\sqrt{m\ln(n/m)B})$.
\qed
\end{proof}
%\Red{
%We conclude this section with a remark on the conditions of $T$ in our lower bounds.
%As noted in Theorem \ref{th:sparse1}, Theorem \ref{th:sparse2} and Theorem \ref{th:dense}, our lower bound only hold when $T$ is large enough.
%, for example, $k$ as in Theorem \ref{th:sparse1}.
%We now argue that these conditions on $T$ are unavoidable.
%To see this, consider for example the PCA setting in Theorem \ref{th:sparse1}. 
%The minimax regret for this setting is trivially upper bounded by $T$ since with sparse instance sequence of length $T$, any algorithm suffers loss at most $T$.
%Thus, for our $\Omega(\sqrt{kT})$ lower bound to hold, $T$ must be asymptotically at least $k$.
%In other words, the minimax regret is actually in $\Theta(\min\{T,\sqrt{kT}\})$ instead of just $\Theta(\sqrt{kT})$.
%We did not give bounds in this form since we think the $\Theta(T)$ bound is trivial and not interesting.
%However, to keep us lower bounds technically correct, we need these conditions on $T$.
%Conditions in Theorem \ref{th:sparse2} and Theorem \ref{th:dense} can be explained similarly.
%}

\section{Conclusion}
\label{sec:conclusion}
We showed in this paper that GD is non-optimal for various problems. 
However, our lower bounds are for the Mirror Descent version of GD 
that trades off the loss on the last example with a divergence 
to the last capped parameter matrix.
There is an alternate algorithm: the Incremental Off-line 
\cite{aw} or Follow the Perturbed Leader algorithm \cite{fpl} 
that in its motivation trades off the total loss on all
examples against the 
divergence to the initial distribution.
Note that both versions follow their update 
with a projection into the parameter space.
We conjecture that the Incremental Off-line version of GD is strictly better
than the commonly studied Mirror Descent version.
The advantage of processing all examples versus
just the last one has now shown up in a number of different contexts:
in Boosting it led to better algorithms \cite{totcorr} 
and it also was crucial for obtaining a
kernelizable online PCA algorithm \cite{kpca}.
When there are only equality constraints
and the loss is linear, then the two versions of the
algorithm are provably the same (See e.g. \cite{perm}).
However when there are inequality constraints that are not
enforced by the divergence, then the projection steps of the
Mirror Descent version of the algorithm ``forgets''
information about the past examples whenever the algorithm
runs into the boundaries of the inequality constraints.

More concretely, we conjecture that the Incremental Off-line
version of GD has the optimal budget regret bound for
online PCA (as Mirror Descent MEG does which enforces
the non-negativity constraints with its divergence). 
If this conjecture is true, then this would be the first case where
there is a provable gap between processing just the last versus
all past examples. If the conjecture is false, then Mirror
Descent MEG is truly better than both versions of GD.
Both outcomes would be in important step forward in our
understanding of online algorithms.
Note that our $k\sqrt{\budget}$ lower bound for 
GD specifically exploits the forgetting effect
and consequently only applies
to the Mirror Descent version of the GD algorithm.

% Define sequence $a_n$ ($1\leq n \leq N$)
% \[
% a_n = \max\{a_1 - (n-1)\eta,0\},
% \]
% where $a_1 \geq 0$, $\eta \geq 0$. If $a_1 - a_N \geq C > 0$,
% the sum $\sum_{n=1}^N a_n = \Omega(\frac{1}{\eta})$.

% Let $a_{N^*}$ be the last non-zero term in the sequence.
% Since $a_1 \dots a_{N^*}$ is an arithmetic sequence, 

% \[
% \sum_{n=1}^{N^*} a_n = \half(a_1 + a_{N^*})(\frac{(a_1 - a_{N^*})}{\eta}+1)  \geq \frac{a_1}{2}(\frac{(a_1 - a_{N^*})}{\eta}+1)
% \]
% If $N^* = N$,
% \[ 
% \frac{(a_1 - a_{N^*})}{\eta}+1 = \frac{(a_1 - a_{N})}{\eta}+1 \geq \frac{C}{\eta}
% \]

% Otherwise $N^* < N$ and $a_{N^*} \leq \eta$, 
% \[
% \frac{(a_1 - a_{N^*})}{\eta}+1 \geq \frac{a_1}{\eta} \geq \frac{C}{\eta}.
% \]

\bibliographystyle{sty/splncs03}
\bibliography{pap}
\end{document}